%% file: main.tex
\documentclass[sigconf]{aamas}  
\input{preamble.tex}

\input{title.tex}
\begin{document}
\maketitle
\input{body.tex}

\appendix
\input{appendix.tex}

\bibliographystyle{ACM-Reference-Format}  
\bibliography{main}  

\end{document}

%% file: preamble.tex
\AtBeginDocument{%
  \providecommand\BibTeX{{%
    \normalfont B\kern-0.5em{\scshape i\kern-0.25em b}\kern-0.8em\TeX}}}

\usepackage{flushend}

\setcopyright{ifaamas}  
\copyrightyear{2020} 
\acmYear{2020} 
\acmDOI{} 
\acmPrice{} 
\acmISBN{} 
\acmConference[AAMAS'20]{Proc.\@ of the 19th International Conference on Autonomous Agents and Multiagent Systems (AAMAS 2020)}{May 9--13, 2020}{Auckland, New Zealand}{B.~An, N.~Yorke-Smith, A.~El~Fallah~Seghrouchni, G.~Sukthankar (eds.)}  

\usepackage[utf8]{inputenc} 
\usepackage[T1]{fontenc}    
\usepackage{hyperref}       
\usepackage{url}            
\usepackage{booktabs}       
\usepackage{amsfonts}       
\usepackage{nicefrac}       
\usepackage{microtype}      

\usepackage{microtype}
\usepackage{graphicx}
\usepackage{subfigure}
\usepackage{booktabs} 
\usepackage{amsmath}
\usepackage{mathtools}
\usepackage{amsthm}
\usepackage{wrapfig}
\usepackage{float}
\mathtoolsset{showonlyrefs}
\allowdisplaybreaks[1]

\newcommand{\vf}{V^\theta_\gamma}
\newcommand{\qf}{Q^\theta_\gamma}
\newcommand{\af}{A^\theta_\gamma}
\newcommand{\df}{d^\theta_\gamma}

\newcommand{\jq}{\nabla J_?(\theta)}

%% file: title.tex
\title{Is the Policy Gradient a Gradient?}

\author{Chris Nota}
\affiliation{%
 \department{College of Information and Computer Sciences}
 \institution{University of Massachusetts Amherst}
}
\email{cnota@cs.umass.edu}

\author{Philip S. Thomas}
\affiliation{%
 \department{College of Information and Computer Sciences}
 \institution{University of Massachusetts Amherst}
}
\email{pthomas@cs.umass.edu}

\begin{abstract}
    The policy gradient theorem describes the gradient of the expected discounted return with respect to an agent's policy parameters.
    However, many policy gradient methods omit the outer discount factor, yielding an update that need not equal the gradient of the discounted objective.
    Is the resulting update the gradient of another objective?
    This has been an open question for several years, and this lack of theoretical clarity has led to an abundance of misstatements in the literature.
    We answer this question for the resulting episode-summed, on-policy update by proving that it need not equal the gradient of any function.
    Further, we construct a counterexample in which exact expected updates drive the policy toward a deterministic limit that is globally \emph{pessimal} with respect to both the discounted and undiscounted objectives, while the parameter diverges.
    We motivate this work by surveying the literature and showing that there remains a widespread misunderstanding regarding discounted policy gradient methods, with errors present even in highly-cited papers published at top conferences.
\end{abstract}


%% file: body.tex
\section{Introduction}
\label{sec:intro}

\emph{Reinforcement learning} (RL) is a subfield of machine learning in which computational \emph{agents} learn to maximize a numerical \emph{reward} signal through interaction with their environment.
\emph{Policy gradient methods} encode an agent's behavior as a parameterized stochastic \emph{policy} and update the policy parameters according to an estimate of the gradient of the expected sum of rewards (the expected  \emph{return}) with respect to those parameters.
In practice, estimating the effect of a particular action on rewards received far in the future can be difficult, so almost all state-of-the-art implementations instead consider an exponentially discounted sum of rewards (the \emph{discounted} return), which shortens the effective time horizon.
The \emph{policy gradient theorem} \cite{sutton2000policy} describes the appropriate update direction for this discounted setting.
However, many modern policy gradient algorithms omit the outer discount factor that appears in the theorem.
It has been an open question for several years whether the resulting update is an unbiased estimate of the gradient of a different, related objective \cite{thomas2014bias}.
We answer this question for the episode-summed, on-policy update: its expectation need not equal the gradient of \emph{any} function.
Further, we show that for some tasks, exact expected updates drive the policy toward a globally pessimal limit under both the discounted and undiscounted objectives, even though the parameter does not converge to a finite value.

Our analysis isolates the effect of omitting the outer discount factor from an on-policy update; applying the result to a particular algorithm requires accounting for its sampling distribution and any other modifications.
In Section \ref{sec:lit-review}, we review all of the policy gradient algorithms included in the popular \texttt{stable-baselines} repository \cite{stable-baselines} and their associated papers, including A2C/A3C \cite{mnih2016asynchronous}, ACER \cite{wang2017sample}, ACKTR \cite{wu2017scalable}, DDPG \cite{lillicrap2015continuous}, PPO \cite{schulman2017proximal}, TD3 \cite{fujimoto2018addressing}, TRPO \cite{schulman2015trust}, and SAC \cite{haarnoja2018sac}.
We motivate this choice in Section \ref{sec:lit-review}, but we note that all of these papers were published at top conferences\footnote{ICML, NeurIPS, or ICLR, with the exception of PPO, which appears to have been published only on arXiv.} and have received hundreds or thousands of citations.
The reviewed implementations omit explicit outer discount weighting, but their complete updates differ through sampling, preconditioning, surrogate objectives, and entropy terms.
While omitting this weighting is a valid algorithmic choice if properly acknowledged, we found that only \emph{one} of the eight papers acknowledged this choice. Two papers made erroneous gradient claims, and the SAC paper overstated the cited justification for an average-reward interpretation; other papers made claims that we found misleading.
The purpose of identifying these errors is not to criticize the authors or the algorithms, but to draw attention to the fact that confusion regarding the behavior of policy gradient algorithms exists at the very core of the RL community and has gone largely unnoticed by reviewers.
This has led to a proliferation of errors in the literature.
We hope that by providing definitive answers to the questions associated with these errors we are able to improve the technical precision of the literature and contribute to the development of a better theoretical understanding of the behavior of RL algorithms.

\noindent\textbf{Revision note (October 2026).} This version includes corrections and clarifications to the mathematical statements, Figure 1, and discussion of prior work. The main result and the original counterexamples are retained.

\section{Notation}
\label{sec:notation}

RL agents learn through interactions with an environment.
An environment is expressed mathematically as a \emph{Markov decision process} (MDP).
An MDP is a tuple, $(\mathcal S, \mathcal A, P, R, d_0, \gamma)$, 
where $\mathcal S$ is the finite set of possible \emph{states} of the environment,
$\mathcal A$ is the finite set of \emph{actions} available to the agent,
$P: \mathcal S \times \mathcal A \times \mathcal S \to [0, 1]$ is a \emph{transition function} that determines the probability of transitioning between states given an action, 
$R: \mathcal S \times \mathcal A \to [-R_\text{max}, R_\text{max}]$ is the expected reward from taking an action in a particular state, bounded by some $R_\text{max} \in \mathbb R$, 
$d_0: \mathcal S \to [0, 1]$ is the \emph{initial state distribution}, and 
$\gamma \in [0, 1]$ is the \emph{discount factor} which decreases the utility of rewards received in the future.
In the \emph{episodic setting}, interactions with the environment are broken into independent \emph{episodes}.
Each episode is further broken into individual \emph{timesteps}.
At each timestep, $t$, the agent observes a state, $S_t$, takes an action, $A_t$, transitions to a new state, $S_{t+1}$, and receives a reward, $R_t$.
Each episode begins with $t=0$ and ends when the agent enters a special state called the \emph{terminal absorbing state}, $s_\infty$.
Once $s_\infty$ is entered, the agent can never leave and receives a reward of 0 forever.
We write $\mathcal S^\circ \coloneqq \mathcal S \setminus \{s_\infty\}$ for the nonterminal states.
We assume almost-sure termination from every nonterminal starting state under each stationary policy considered.
Since the state space is finite, the termination time $T \coloneqq \inf\{t \geq 0 : S_t = s_\infty\}$ then has finite expectation:
\begin{equation*}
    \mathbb E[T \mid \theta] = \sum_{t=0}^{\infty} \Pr(S_t \neq s_\infty \mid \theta) < \infty.
\end{equation*}

A \emph{policy}, $\pi: \mathcal S \times \mathcal A \to [0, 1]$, determines the probability that an agent will choose an action in a particular state.
A \emph{parameterized policy}, $\pi^\theta$, is a policy that is defined as a function of some parameter vector, $\theta$, which may be the weights in a neural network, values in a tabular representation, etc.
For our analysis, $\pi^\theta(s,a)$ is positive and twice continuously differentiable in $\theta$ on an open parameter domain; the MDP components are independent of $\theta$.
The \emph{compatible features} of a parameterized policy represent how $\theta$ may be changed in order to make a particular action, $a \in \mathcal A$, more likely in a particular state, $s \in \mathcal S$, and are defined as $\psi(s, a) \coloneqq \frac{\partial}{\partial \theta} \ln \pi^\theta(s, a)$.
The \emph{value function}, $\vf: \mathcal S \to \mathbb R$, represents the expected discounted sum of rewards when starting in a particular state under policy $\pi^\theta$; that is, $\forall t, \vf(s) \coloneqq \mathbb E[\sum_{k = 0}^\infty \gamma^k R_{t + k} | S_t {=} s, \theta]$, where conditioning on $\theta$ indicates that $\forall t, A_t \sim \pi^\theta(S_t, \cdot)$.
The \emph{action-value function}, $\qf: \mathcal S \times \mathcal A \to \mathbb R$, is similar, but also considers the action taken; that is, $\forall t, \qf(s, a) \coloneqq \mathbb E[\sum_{k = 0}^\infty \gamma^k R_{t + k} | S_t {=} s, A_t {=} a, \theta]$.
The \emph{advantage function} is the difference between the action-value function and the (state) value function: $\af(s, a) \coloneqq \qf(s, a) - \vf(s)$.

The \emph{objective} of an RL agent is to maximize some function, $J$, of its policy parameters, $\theta$.
In the episodic setting, the two most commonly stated objectives are the \emph{discounted objective}, $J_\gamma(\theta) = \mathbb E[\sum_{t = 0}^\infty \gamma^t R_t | \theta]$, and the \emph{undiscounted objective}, $J(\theta) {=} \mathbb E[\sum_{t = 0}^\infty R_t | \theta]$.
The discounted objective has some convenient mathematical properties, but it corresponds to few real-world tasks.
\citet{sutton2018reinforcement} have even argued for its deprecation.
However, we will see in Section \ref{sec:lit-review} that the discounted objective is commonly stated as a justification for the use of a discount factor, even when the connection between that objective's gradient and the implemented update is not established.

\section{Problem Statement}

\label{sec:problem-statement}

The formulation of the \emph{policy gradient theorem} \cite{williams1992simple, sutton2000policy, Baxter2000DirectGR} presented by \citet{sutton2000policy} was given for two objectives: the average reward objective for the infinite horizon setting \cite{mahadevan1996average} and the discounted objective, $J_\gamma$, for the episodic setting.
The episodic setting considered in this paper is more popular, as it is better suited to the types of tasks that RL researchers typically use for evaluation (e.g., many classic control tasks, Atari games \cite{bellemare2013arcade}, etc.).
The discounted \emph{policy gradient}, $\nabla J_\gamma(\theta)$, tells us how to modify the policy parameters, $\theta$, in order to increase $J_\gamma$, and is given by:

\begin{equation}
    \label{eq:discounted-pg}
    \nabla J_\gamma(\theta) = \mathbb E \left[\sum_{t=0}^\infty \gamma^t \psi^\theta(S_t, A_t) \qf(S_t, A_t) \bigg| \theta \right].
\end{equation}
Because $\nabla J_\gamma$ is the true gradient of the discounted objective, stochastic gradient ascent using unbiased estimates of it admits standard convergence guarantees.
If $J_\gamma$ is bounded above and suitable smoothness, step-size, and conditional noise assumptions hold, then the gradient converges to zero almost surely, and every finite accumulation point of the parameter sequence is stationary \cite{Bertsekas2000}.
Omitting the outer discount factor from the on-policy, episode-summed estimator instead gives an estimator of the expression:

\begin{equation}
    \label{eq:biased_pg}
    \nabla J_?(\theta) = \mathbb E \left[\sum_{t=0}^{\infty} \psi^\theta(S_t, A_t) \qf(S_t, A_t) \bigg| \theta \right].
\end{equation}
%
Note that this expression includes the $\gamma^k$ contained in $\qf$, but differs from the true discounted policy gradient in that it drops the outer $\gamma^t$.
We label this expression $\nabla J_?(\theta)$ because the question of whether or not it is the gradient of some objective function, $J_?$, was left open by \citet{thomas2014bias}.
\citet{thomas2014bias} gives a conditional average-reward natural-gradient interpretation in the continuing setting, but this does not construct a potential for the episode-summed field considered here.
The goal of this paper is to provide answers to the following questions:

\begin{itemize}
    \item Is $\jq$ the gradient of some objective function?
    \item If not, do iterates generated by $\nabla J_?(\theta)$ at least converge in policy to a reasonable limit?
\end{itemize}


\section{$\jq$ is Not a Gradient in General}
\label{sec:not-gradient}

In this section, we answer the first of our two questions and show that the episode-summed update field $\jq$ is not, in general, the gradient of any function using a proof by contraposition with the Clairaut-Schwarz theorem on mixed partial derivatives \cite{shwartz1873asymmetry}.
First, we present this theorem (Theorem \ref{thm:clairaut-shwarz}) and its contrapositive (Corollary \ref{cor:clairaut-shwarz}).
Next, we present Lemma \ref{lem:biased-dist}, which allows us to rewrite $\nabla J_?(\theta)$ in a new form.
Finally, in Theorem \ref{thm:no-grad} we apply the mixed-partial symmetry criterion and Lemma \ref{lem:biased-dist} and derive a counterexample proving that $J_?$ does not, in general, exist,
and therefore that the ``policy gradient'' given by $\jq$ need not be a gradient.
For a fixed MDP and policy parameterization, let $F_\gamma(\theta)$ denote the expression defining $\jq$ in Section \ref{sec:problem-statement}.
This is the same update field denoted by $\jq$; the notation $F_\gamma$ does not presume the existence of an objective.

\begin{theorem}
    \label{thm:clairaut-shwarz}
    \textbf{(Clairaut-Schwarz theorem)}: If $f: \mathbb R^n \to \mathbb R$ exists and is continuously twice differentiable in some neighborhood of the point $(a_1, a_2, \dots, a_n)$, then its second derivative is symmetric:
    \begin{equation}
        \forall i, j: \frac{\partial^2 f(a_1, a_2, \dots, a_n)}{\partial x_i \partial x_j} = \frac{\partial^2 f(a_1, a_2, \dots, a_n)}{\partial x_j \partial x_i}.
    \end{equation}
\end{theorem}
\begin{proof}
    The first complete proof was given by Schwarz \cite{shwartz1873asymmetry}.
    English proofs can be found in many advanced calculus and analysis textbooks \cite[p.~236]{rudin1964principles}.
\end{proof}

\begin{corollary}
    \label{cor:clairaut-shwarz}
    \textbf{(Contrapositive of Clairaut-Schwarz)}: If at some point $(a_1, a_2, \dots, a_n)$  $\in \mathbb R^n$ there exists an $i$ and $j$ such that
    \begin{equation}
        \frac{\partial^2 f(a_1, a_2, \dots, a_n)}{\partial x_i \partial x_j} \neq \frac{\partial^2 f(a_1, a_2, \dots, a_n)}{\partial x_j \partial x_i},
    \end{equation}
    then $f$ does not exist or is not continuously twice differentiable in any neighborhood of $(a_1, a_2, \dots, a_n)$.
\end{corollary}
\begin{proof}
    Contrapositive of Theorem \ref{thm:clairaut-shwarz}.
    As a reminder, the contrapositive of a statement, $P \implies Q$, is $\neg Q \implies \neg P$.
    The contrapositive is always implied by the original statement.
    Additionally, recall that for any function $g$, $\neg \forall i: g(i) \implies \exists i: \neg g(i)$.
\end{proof}

If $F_\gamma$ is continuously differentiable and its Jacobian is asymmetric at a point, it cannot equal the gradient of a differentiable function on a neighborhood of that point.
Indeed, such a function would have continuously differentiable first partial derivatives and hence be twice continuously differentiable, contradicting Theorem \ref{thm:clairaut-shwarz}.
To apply this criterion, we present a lemma that rewrites $\jq$ in a form convenient for differentiating the update field.
The result of this lemma is of some theoretical interest in itself, but further interpretation is left as future work.
We do not leverage it here for any purpose except to aid in our proof of Theorem \ref{thm:no-grad}.

\begin{lemma}
\label{lem:biased-dist}
For $s \in \mathcal S^\circ$, let $\df$ be the unnormalized, weighted state distribution given by:
\begin{equation}
    \label{eq:dist}
    d^\theta_\gamma(s) \coloneqq d_0(s) + (1 - \gamma) \sum_{t = 1}^{\infty} \Pr(S_t = s | \theta).
\end{equation}
Then:
\begin{equation}
    \label{eq:sud}
    \nabla J_?(\theta) = \sum_{s \in \mathcal S^\circ} d^\theta_\gamma(s) \frac{\partial}{\partial \theta } \vf(s).
\end{equation}
\end{lemma}
\begin{proof}
See appendix.
\end{proof}

Lemma \ref{lem:biased-dist} reveals how the dependence of $\df$ on $\theta$ can make the Jacobian of $F_\gamma$ asymmetric:
\begin{align*}
    \frac{\partial F_{\gamma,j}(\theta)}{\partial \theta_i}
    &= \underbrace{\sum_{s \in \mathcal S^\circ} d^\theta_\gamma(s) \frac{\partial^2 \vf(s)}{\partial \theta_i \partial \theta_j}}_{\text{symmetric in }i,j} \\
    &\quad + \underbrace{\sum_{s \in \mathcal S^\circ} \frac{\partial d^\theta_\gamma(s)}{\partial \theta_i} \frac{\partial \vf(s)}{\partial \theta_j}}_{\text{potentially asymmetric in }i,j}.
\end{align*}
The following counterexample makes the second term asymmetric everywhere for $\gamma<1$.

\begin{theorem}
\label{thm:no-grad}
For every $\gamma \in [0,1)$, there exist a finite episodic MDP with discount factor $\gamma$ and a positive, smooth policy parameterization $\pi^\theta$, $\theta \in \mathbb R^2$, for which $F_\gamma$ is not a gradient, even locally: no differentiable function $J:U\to\mathbb R$ on any nonempty open set $U\subseteq\mathbb R^2$ satisfies
\begin{equation*}
    \nabla J(\theta)=F_\gamma(\theta) \qquad\text{for every }\theta\in U.
\end{equation*}
Moreover, the policy parameterization and all MDP components other than $\gamma$ can be chosen independently of $\gamma$.
\end{theorem}
\begin{proof}
Use the MDP in Figure \ref{fig:asymmetric_hessian}, with discount factor $\gamma$ and the policy
\begin{equation*}
    \pi^\theta(s_i,a_1)=\sigma(\theta_i)=1-\pi^\theta(s_i,a_2),\quad i=1,2,
\end{equation*}
where $\sigma(x)=1/(1+e^{-x})$; at the terminal state, use a fixed uniform policy.
This parameterization is positive and smooth on $\mathbb R^2$.
The rewards are 0 or 1, and every episode terminates within two steps, so the standing assumptions hold.
None of the transition probabilities, rewards, initial-state probabilities, or policy probabilities depends on $\gamma$.

Applying Lemma \ref{lem:biased-dist} gives
\begin{align*}
    F_{\gamma,1}(\theta)&=\gamma\sigma'(\theta_1)\sigma(\theta_2),\\
    F_{\gamma,2}(\theta)&=\sigma(\theta_1)\sigma'(\theta_2).
\end{align*}
The appendix supplies the intermediate value-function and derivative calculations.
Consequently,
\begin{align*}
    \frac{\partial F_{\gamma,2}}{\partial\theta_1}
    -\frac{\partial F_{\gamma,1}}{\partial\theta_2}
    &=(1-\gamma)\sigma'(\theta_1)\sigma'(\theta_2)\\
    &>0
\end{align*}
for every $\theta\in\mathbb R^2$ and $\gamma\in[0,1)$.
If a differentiable $J$ satisfied $\nabla J=F_\gamma$ on a nonempty open set $U$, smoothness of $F_\gamma$ would make $J$ twice continuously differentiable on $U$.
Theorem \ref{thm:clairaut-shwarz} would then require the two cross derivatives to agree on $U$, a contradiction.\qedhere
\end{proof}

At $\gamma=1$, the update field $F_1$ coincides with the true undiscounted policy gradient; in this example, $F_1=\nabla[\sigma(\theta_1)\sigma(\theta_2)]$.
For $\gamma<1$, the update can still be a gradient in particular MDPs, such as one-step episodes, where dropping the outer discount changes nothing.
Theorem \ref{thm:no-grad} therefore establishes failure \emph{in general}, not in every MDP.

\begin{figure}[t]
    \centering
    \includegraphics[height=60pt]{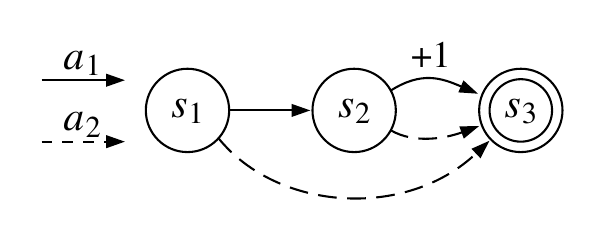}
    \caption{
        The counterexample used in Theorem \ref{thm:no-grad}, whose update field $F_\gamma$ has an asymmetric Jacobian for all $\gamma \in [0,1)$.
        The agent begins in $s_1$, and may choose between actions $a_1$ and $a_2$, each of which produces deterministic transitions; $s_3=s_\infty$ is terminal.
        The rewards are 0 everywhere except when the agent chooses $a_1$ in state $s_2$, which produces a reward of $+1$.
        The policy is the sigmoid parameterization specified in the proof: $\pi^\theta(s_i,a_1)=\sigma(\theta_i)$ for $i=1,2$.
    }
    \label{fig:asymmetric_hessian}
\end{figure}

\section{$\jq$ Can Yield a Pessimal Limiting Policy}

\label{sec:pessimal}

Having established that, for every $\gamma < 1$, $\jq$ need not be the gradient of any function, we ask whether iterates generated by this update at least converge in policy to a reasonable limit.
Failure to follow a gradient does not, by itself, rule out useful convergence guarantees.
For example, the expected update of conventional temporal-difference learning, TD(0), is not generally a gradient field \cite{sutton2015introduction, sutton2009fast}.
Nevertheless, for discounted on-policy prediction on a finite ergodic Markov chain with linear function approximation, TD(0) converges under standard feature and step-size assumptions to the unique projected Bellman fixed point \cite{tsitsiklis1997analysis}.
This point attains zero mean-squared projected Bellman error (MSPBE) \cite{sutton2009fast}, and its value-function error is bounded by a discount-dependent multiple of the smallest error achievable with the given features, measured in the stationary-distribution-weighted norm \cite{tsitsiklis1997analysis}.
These guarantees do not extend generally to nonlinear function approximation, where TD can diverge \cite{tsitsiklis1997analysis}.

We ask whether the update $\jq$ admits an analogous guarantee for its limiting policies.
While it is not clear in what sense a limiting policy should be ``reasonable,'' we propose a very minimal criterion: it should at least surpass the worst possible (pessimal) policy under \emph{either} the discounted or undiscounted objective.
A limiting policy that is pessimal under both objectives would fail even this low bar.

To demonstrate this, we contrast two examples, given in Figures \ref{fig:short-vs-long} and \ref{fig:degenerate}.
In the former example, the limiting policy is optimal under the discounted objective but pessimal under the undiscounted objective.
This is a well understood trade-off of discounting  that can be explained as ``short-sightedness'' by the agent.
The latter example shows that exact expected updates can approach the smallest possible discounted return \emph{and} undiscounted return, even though the opposite deterministic limit maximizes both.
We therefore suggest that for at least some MDPs, $\jq$ may not be a reasonable choice.

To make these convergence claims precise, consider exact expected updates
\begin{equation*}
    \theta_{k+1}=\theta_k+\alpha_k F_\gamma(\theta_k),
    \qquad \alpha_k>0,\quad \sum_{k=0}^\infty\alpha_k=\infty,
\end{equation*}
with finite $\theta_0$ and finite step sizes.
Write $p=\sigma(\theta)$ for the probability of choosing $a_1$, where $\sigma(\theta)=1/(1+e^{-\theta})$.
In Figure \ref{fig:short-vs-long}, only the policy at $s_1$ affects returns. At $\gamma=1/2$, we have $J_{1/2}=\tfrac12+\tfrac12p$, $J_1=2-p$, and $F_{1/2}=\tfrac12p(1-p)>0$.
Consequently, $\theta_k\to+\infty$ and $p_k\to1$, yielding the stated optimal discounted and pessimal undiscounted limits.

In Figure \ref{fig:degenerate}, the same parameter controls the policy in every state and $\gamma=0$.
Both $s_1$ and $s_2$ are visited in every episode; their expected update contributions are $2p(1-p)$ and $-4p(1-p)$, respectively.
Actions in the remaining states do not affect rewards, so
\begin{equation*}
    F_0(\theta)=-2p(1-p)<0
    \qquad\text{for every finite }\theta.
\end{equation*}
Thus $\theta_k$ decreases monotonically.
If it had a finite limit, $p_k(1-p_k)$ would remain bounded away from zero, contradicting $\sum_k\alpha_k=\infty$.
Hence $\theta_k\to-\infty$ and $p_k\to0$; the same argument with the sign reversed proves the claim for Figure \ref{fig:short-vs-long}.
The returns in Figure \ref{fig:degenerate} are
\begin{equation*}
    J_0(\theta)=2p-1,\qquad J_1(\theta)=198p-99,
\end{equation*}
and therefore tend to $-1$ and $-99$, respectively.
These are global minima over policies: the initial reward is at least $-1$, and the remaining two rewards sum to either $98$ or $-98$.
The deterministic policy that always chooses $a_2$ attains both minima, but it lies on the boundary of the sigmoid policy class and is not attained at any finite $\theta$.
The result is convergence in policy to a pessimal limit, not convergence to a finite parameter fixed point.

\begin{figure}
\begin{subfigure}
    \centering
    \includegraphics[width=0.8\columnwidth]{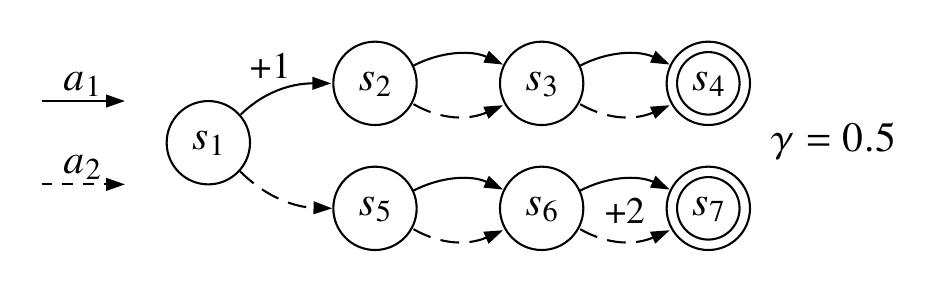}
    \caption{
        An example where exact expected updates approach a policy that is optimal with respect to the discounted objective but pessimal with respect to the undiscounted objective.
        The agent starts in state $s_1$, and can achieve a reward of $+1$ by transitioning from $s_1$ to $s_2$, and a reward of $+2$ by transitioning from $s_6$ to $s_7$, regardless of the action taken.
        In order to maximize the undiscounted return, the agent should choose $a_2$.
        However, the discounted return is higher from choosing $a_1$.
        Only the return in $s_1$ will affect the policy gradient as the advantage is zero in every other state.
        Under the sigmoid parameterization and step-size assumptions in Section \ref{sec:pessimal}, the probability of choosing $a_1$ at $s_1$ converges to one as $\theta\to+\infty$.
        If the researcher is concerned with the \emph{undiscounted} objective, as is often the case, this result is problematic.
        Choosing a larger value of $\gamma$ trivially fixes the problem in this particular example, but for any value of $\gamma < 1$, similar problems will arise given a sufficiently long horizon, and thus the problem is not truly eliminated.
        Nevertheless, this is well-understood as the trade-off of discounting.
    }
    \label{fig:short-vs-long}
\end{subfigure}

\begin{subfigure}
    \centering
    \includegraphics[width=\columnwidth]{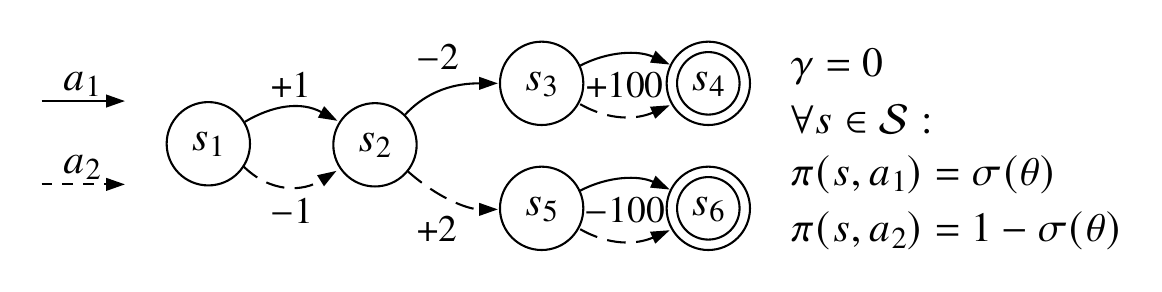}
    \caption{
        An example where exact expected updates approach a policy that is pessimal with respect to both the discounted and undiscounted objectives.
        In this formulation, there is a single policy parameter, $\theta$, so the agent must execute the same policy in every state.
        It is difficult to justify any solution other than to always choose $a_1$.
        If the agent is completely myopic, then $a_1$ gives the superior immediate reward of $+1$ in the starting state.
        If the agent is somewhat farther-sighted, then always choosing $a_1$ will eventually result in  the $+100$ reward.
        Choosing $a_2$ provides no benefit with respect to either view.
        Under the step-size assumptions in Section \ref{sec:pessimal}, following the exact expected update $\jq$ makes the probability of choosing $a_2$ converge to one as $\theta\to-\infty$.
        The expected update contribution from $s_2$ is twice as large in magnitude as that from $s_1$ and has the opposite sign; the contributions from $s_3$ and $s_5$ are zero.
        Because $\jq$ ignores the change in the state distribution, the net expected update increases the probability of choosing $a_2$ at every finite parameter value.
        Again, while we chose $\gamma = 0$ for simplicity, similar examples can be produced in long-horizon problems with ``reasonable'' settings of $\gamma$, such as $0.99$ or higher.
        One may ask if the sharing of a policy between states is contrived, but such a situation occurs reliably under partial observability or when a function approximator is used.
    }
    \label{fig:degenerate}
\end{subfigure}
\end{figure}

\section{Literature Review}
\label{sec:lit-review}

We now examine how a subset of popular policy gradient papers describes the relationship between objectives and updates that omit the outer discount factor.
We show that the majority of them include erroneous or misleading statements directly relating to arguments made in this paper.
We emphasize that we do not pose this as a criticism of the authors, but rather a symptom of the ambiguity in the literature that we hope to address with this paper.

Theorem \ref{thm:no-grad} concerns the expected episode-summed update in episodic tasks under the current policy, with exact action values and no preconditioning or entropy terms.
It does not imply the same conclusion for a natural-gradient update, a replay-weighted update, or an update with entropy terms.
A field can increase an objective without equaling its Euclidean gradient.
Nor does it preclude gradients of fixed-batch surrogate objectives, such as those used by PPO and TRPO: a surrogate with its data and value estimates held fixed is distinct from the field obtained as the sampling policy changes.
We therefore distinguish the gradient identities stated in the papers from claims about their complete algorithms.

\subsection{Methodology} 

Rather than manually selecting which papers to review, which may have introduced an unacceptable degree of bias, we chose to review the papers associated with every policy gradient algorithm included in  \texttt{stable-baselines} \cite{stable-baselines}, a fork of the popular OpenAI \texttt{baselines} \cite{baselines} library.
We chose this particular subset of algorithms because inclusion in the library generally indicates that the algorithms have achieved a certain level of popularity and relevance.
The papers corresponding to these algorithms have received hundreds or thousands of citations each, 
and, with the exception of PPO \cite{schulman2017proximal}, were published at top machine learning conferences.
It therefore seems reasonable to claim that the papers were impactful in the field and are representative of high-quality research.
While we acknowledge that this sampling of algorithms is heavily biased towards the subfield of ``deep'' RL, we argue that this is not unreasonable given the immense popularity of this area and its impact on the broader field.

For each algorithm, we examined the pseudocode for the algorithm itself, the background and theoretical sections of the associated paper, and several publicly available implementations, including the implementation created by the authors where available.
We tried to answer the following three questions for each algorithm:

\begin{enumerate}
    \item Does the actor update omit explicit outer discount weighting?
    \item If so, did the authors acknowledge this departure from the discounted policy gradient expression?
    \item Did the paper include any erroneous or misleading claims connecting its updates to an objective?
\end{enumerate}
For questions (2) and (3), we support our evaluation of the papers with quotations from the text in cases where there were errors or ambiguities.
While this approach is verbose, we felt that paraphrasing the original papers would not allow the readers to understand the errors in their appropriate context.
For 5 out of 8 of the algorithms,\footnote{ACKTR, PPO, TD3, TRPO, and SAC} the original authors or organizations provided code, allowing us to directly answer (1).
For all eight papers, we examined the \texttt{stable-baselines} \cite{stable-baselines} implementation as well as several other implementations including \texttt{tf-agents} \cite{TFAgents}, \texttt{garage} \cite{garage}, \texttt{spinning-up} \cite{SpinningUp2018}, and the \texttt{autonomous-learning-library} \cite{nota2020autonomous}.
Finally, for each paper we note the conference, year, and citation count estimated by Google Scholar on February 23, 2020.

\subsection{Results}

Eight policy gradient algorithms appear in \texttt{stable-baselines}.
Our high-level results to each question are as follows:

\begin{enumerate}
    \item All eight reviewed algorithms omit explicit outer discount weighting in their actor updates. This shared omission does not identify all eight complete updates with $\jq$.
    \item Only \emph{one} out of the eight papers calls attention to this departure from the discounted policy gradient expression.
    \item Two out of the eight papers included explicitly erroneous gradient claims. A third, SAC, overstated the cited justification for an average-reward interpretation. Two additional papers made misleading claims in that they presented the discounted policy gradient ($\nabla J_\gamma(\theta)$), and then proposed updates without explicit outer discount weighting, without clarifying the relationship between the two.
    A sixth paper included claims that we argue are misleading, but not strictly false.
\end{enumerate}
We reproduce the two explicitly erroneous gradient claims and the overstated average-reward claim below in order to give the reader a sense of common misunderstandings.
Notice that the quotations sometimes use slightly different notation than this paper.
We try to supply the appropriate context such that the meaning intended by the authors can be understood.
Additional analysis can be found in the appendix.

\subsection*{A3C (ICML 2016, 2859 citations)}

Asynchronous Advantage Actor-Critic (A3C) \cite{mnih2016asynchronous} is an actor-critic method that generates sample batches by running multiple actors in parallel.
The original version of the algorithm achieved state-of-the-art performance on many Atari games when it was published.
The background section of the main paper includes the text:

\begin{quote}
    The return $R_t = \sum_{k=0}^\infty \gamma^k r_{t+k}$ is the total accumulated return from time step $t$ with discount factor $\gamma \in (0, 1]$. The goal of the agent is to maximize the expected return from each state $s_t$.
    
    [\dots]
    
    In contrast to value-based methods, policy based methods directly parameterize the policy $\pi(a | s; \theta)$ and update the parameters $\theta$ by performing, typically approximate, gradient ascent on $\mathbb E[R_t]$. One example of such a method is the REINFORCE family of algorithms due to \citet{williams1992simple}. Standard REINFORCE updates the policy parameters $\theta$ in the direction $\nabla_\theta \log \pi(a_t|s_t; \theta)R_t$, which is an unbiased estimate of $\nabla_\theta \mathbb E[R_t]$.
    It is possible to reduce the variance of this estimate while keeping it unbiased by subtracting a learned function of the state $b_t(s_t)$, known as a baseline \cite{williams1992simple}, from the return. The resulting gradient is $\nabla_\theta \log \pi(a_t|s_t;\theta)(R_t - b_t(s_t))$.
\end{quote}
It is falsely claimed that $\nabla_\theta \log \pi(a_t | s_t;\theta)R_t$ (whose sum over an episode is an unbiased estimator of $\jq$) is an unbiased estimate of $\nabla_\theta \mathbb E[R_t]$, where $R_t$ is the discounted return at timestep $t$ rather than the individual reward.
We showed that this is not the case.

\subsection*{ACKTR (NeurIPS 2017, 235 citations)}

Actor Critic using Kronecker-Factored Trust Region (ACKTR) \cite{wu2017scalable} is a variation of A3C that attempts to efficiently estimate the \emph{natural} policy gradient, which uses an alternate notion of the direction of steepest ascent.
The resulting algorithm is considerably more sample efficient than A3C.
The background section contains:

\begin{quote}
    The goal of the agent is to maximize the expected $\gamma$-discounted cumulative return $J(\theta) = \mathbb E_\pi[R_t] = \mathbb E_\pi[\sum_{i\geq0}^\infty \gamma^i r(s_{t+i}, a_{t+i})]$ with respect to the policy parameters $\theta$.
    Policy gradient methods directly parameterize a policy $\pi_\theta(a|s_t)$ and update parameter $\theta$ so as to maximize the objective $J(\theta)$. In its general form \cite{schulman2015high}, the policy gradient is defined as:
    
    \begin{equation}
        \nabla_\theta J(\theta) = \mathbb E_\pi \left[\sum_{t=0}^\infty \Psi^t \nabla_\theta \log \pi_\theta(a_t|s_t)\right],
    \end{equation}
    where $\Psi^t$ is often chosen to be the advantage function $A^\pi(s_t, a_t)$, which provides a relative measure of value of each action $a_t$ at a given state $s_t$.
\end{quote}
The authors assert that the gradient, $\nabla_\theta J(\theta)$, is equal to the given expression, which is false for the proposed choice of $A^\pi(s_t,a_t)$.
They did not note that the above expression is an approximation or otherwise clarify.
Therefore, the definition of the policy gradient presented above is erroneous.
This criticism concerns the displayed unpreconditioned gradient identity, not the subsequent natural-gradient preconditioning used by ACKTR.

\subsection*{SAC (ICML 2018, 446 citations)}

Soft Actor-Critic (SAC) \cite{haarnoja2018sac} is an off-policy maximum-entropy actor-critic algorithm with a stochastic actor.
Its actor update uses replay-buffer states and an entropy term, so it is not the episode-summed field $\jq$ analyzed here.
We examine the following general statement about discounting in its appendix, rather than identify SAC's complete update with $\jq$.
The appendix states:

\begin{quote}
    The exact definition of the discounted maximum entropy objective is complicated by the fact that, when using a discount factor for policy gradient methods, we typically do not discount the state distribution, only the rewards. In that sense, discounted policy gradients typically do not optimize the true discounted objective. Instead, they optimize average reward, with the discount serving to reduce variance, as discussed by \citet{thomas2014bias}.
\end{quote}
The cited analysis does not establish this in general: for non-episodic tasks, \citet{thomas2014bias} gives a sufficient condition under which biased natural actor-critic updates follow the direction of the average-reward natural gradient.
Theorem \ref{thm:no-grad} concerns the episodic setting, where the episode-summed update $\jq$ need not equal the gradient of any objective.
It does not directly address SAC's infinite-horizon, replay-based update and does not, by itself, establish whether SAC optimizes average reward.

\subsection{Discussion}

These passages illustrate the need to distinguish an update being an objective's gradient from broader claims that it optimizes that objective.
Two of the remaining papers did not establish how their implemented updates relate to the gradient of the stated discounted objective, $\nabla J_\gamma(\theta)$.
Even among the papers where we did not find explicit errors, errors were avoided largely through the use of hedged language and ambiguity, rather than technical precision.
For examples of this, we encourage the reader to refer to the appendix.
While for the purposes of this review we only sampled a small subset of the literature on policy gradients, we found the results sufficient to support our claim that there exists a widespread misunderstanding regarding updates that omit the outer discount factor.

\section{Conclusions}

We conclude that the episode-summed update $\jq$ need not equal any objective's gradient (Section \ref{sec:not-gradient}), and exact expected updates can drive policies toward a pessimal limit while the parameter diverges (Section \ref{sec:pessimal}).
Our literature review documents recurring misstatements about the relationship between updates that omit the outer discount factor and their stated objectives.
Nevertheless, methods that omit this factor have remained effective in practical applications.
The precise reason for this effectiveness, especially in the episodic setting, remains an open question.

%% file: appendix.tex
\section*{Appendix}

\subsection*{Proof of Lemma \ref{lem:biased-dist}}

We first justify the differentiation and summation steps below.
Let $P_\theta$ be the policy-induced transition matrix restricted to $\mathcal S^\circ$.
Almost-sure absorption in a finite state space makes $P_\theta$ transient, so $I-\gamma P_\theta$ is nonsingular for every $\gamma \in [0,1]$.
The finite-dimensional Bellman equation and the policy smoothness assumptions therefore imply that $\vf(s)$ is twice continuously differentiable; the same holds for the nonterminal occupancy weights obtained from $(I-P_\theta)^{-1}$.
At each fixed $\theta$, finiteness of the state and action spaces and positivity of the policy imply a finite bound $C_\theta$ on $\|\qf(s,a)\psi(s,a)\|$.
Since this product is zero at the terminal state,
\begin{equation*}
    \mathbb E\left[\sum_{i=0}^{\infty}\|\qf(S_i,A_i)\psi(S_i,A_i)\|\,\middle|\,\theta\right]
    \leq C_\theta\mathbb E[T\mid\theta] < \infty.
\end{equation*}
For the nonnegative weights derived below, $\sum_{t=0}^i w(t)\gamma^{i-t}=1$.
Thus the corresponding double series is also absolutely integrable, justifying the exchanges of sums and expectations.

We begin by hypothesizing that $\jq$ takes the form of a weighted distribution over $\frac{\partial}{\partial \theta } \vf(s)$, given some time-dependent weights, $w(t)$, on each term in the state distribution.
That is, we hypothesize that equality:
\begin{equation}
     \nabla J_?(\theta) = \sum_{s \in \mathcal S^\circ} d^\theta_\gamma(s) \frac{\partial}{\partial \theta } \vf(s),
\end{equation}
holds for some $\df$, defined on $\mathcal S^\circ$:
\begin{align*}
    \df(s) = \sum_{t = 0}^{\infty} w(t) \Pr(S_t = s | \theta).
\end{align*}
We then must prove that this holds for some choice of $w(t)$, and then derive the satisfying choice of $w(t)$.
\citet{sutton2000policy} established that:
\begin{equation}
\label{eq:lkjahwglkjag}
   \frac{\partial}{\partial \theta} \vf(s){=}\sum_{k=0}^{\infty} \sum_{x \in \mathcal S^\circ}\gamma^k \Pr(S_{t+k}{=}x|S_t{=}s,\theta)\sum_{a \in \mathcal A} \qf(x,a) \frac{\partial \pi^\theta(x,a)}{\partial \theta}.
\end{equation}
Substituting this into our expression for $\jq$ gives us:
\begin{align*}
    & \sum_{s \in \mathcal S^\circ} d^\theta_\gamma(s) \frac{\partial}{\partial \theta } \vf(s) \\
    =& \sum_{s \in \mathcal S^\circ} \sum_{t = 0}^{\infty} w(t) \Pr(S_t = s | \theta) \\
    &\times \sum_{k=0}^{\infty}\sum_{x \in \mathcal S^\circ} \gamma^k \Pr(S_{t + k} = x | S_t = s, \theta) \sum_{a \in \mathcal A} \qf(x,a) \frac{\partial \pi^\theta(x,a)}{\partial \theta} \\
    =& \sum_{s \in \mathcal S^\circ} \sum_{t = 0}^{\infty} \sum_{k=0}^{\infty}\sum_{x \in \mathcal S^\circ} \sum_{a \in \mathcal A}  w(t) \gamma^k \\
    &\times \Pr(S_t = s | \theta) 
     \Pr(S_{t + k} = x | S_t = s, \theta) \qf(x,a) \frac{\partial \pi^\theta(x,a)}{\partial \theta}\\
    =& \sum_{s \in \mathcal S^\circ} \sum_{t = 0}^{\infty} \sum_{k=0}^{\infty}\sum_{x \in \mathcal S^\circ} \sum_{a \in \mathcal A}  w(t) \gamma^k \\
    &\times \Pr(S_t = s | \theta) \Pr(S_{t + k} = x | S_t = s, \theta) \\
    &\times \qf(x, a) \pi^\theta(x,a)\frac{\partial \ln (\pi^\theta(x,a))}{\partial \theta} \\
    =& \sum_{s \in \mathcal S^\circ} \sum_{t = 0}^{\infty} \sum_{k=0}^{\infty}\sum_{x \in \mathcal S^\circ} \sum_{a \in \mathcal A}  w(t) \gamma^k \\
    &\times \Pr(S_t = s | \theta) \Pr(S_{t + k} = x | S_t = s, \theta) \qf(x, a) \pi^\theta(x,a) \psi(x, a) \\
    =& \sum_{s \in \mathcal S^\circ} \sum_{t = 0}^{\infty} \sum_{k=0}^{\infty}\sum_{x \in \mathcal S^\circ} \sum_{a \in \mathcal A}  w(t) \gamma^k \\
    &\times \Pr(S_t = s | \theta)
     \Pr(S_{t + k} = x | S_t = s, \theta) \Pr(A_{t+k} = a | S_{t + k} = x, \theta) \\
    &\times \qf(x, a) \psi(x, a) \\
    =& \sum_{t = 0}^{\infty} \sum_{k=0}^{\infty}\sum_{x \in \mathcal S^\circ} \sum_{a \in \mathcal A}  w(t) \gamma^k \\
    &\times \Pr(S_{t + k} = x | \theta) \Pr(A_{t+k} = a | S_{t + k} = x, \theta) \qf(x, a) \psi(x, a), \\
\end{align*}
since $\Pr(A_{t+k}=a |S_{t+k}=x,\theta)=\Pr(A_{t+k}=a |S_{t+k}=x,S_t=s,\theta)$ and by the law of total probability; for $x \in \mathcal S^\circ$, a trajectory ending at $x$ cannot have previously entered $s_\infty$.
The key point is that we have removed the term $\Pr(S_t = s | \theta)$.
Continuing, starting with the fact that $\Pr(S_{t+k}=x |  \theta) \Pr(A_{t+k}=a |S_{t+k}=x,\theta)= \Pr(S_{t+k}=x,A_{t+k}=a |\theta)$, we have that:

\begin{align*}
    & \sum_{s \in \mathcal S^\circ} d^\theta_\gamma(s) \frac{\partial}{\partial \theta } \vf(s) \\
    =& \sum_{t = 0}^{\infty} \sum_{k=0}^{\infty}\sum_{x \in \mathcal S^\circ} \sum_{a \in \mathcal A}  w(t) \gamma^k \\
    &\times \Pr(S_{t + k} = x, A_{t+k} = a | \theta) \qf(x, a) \psi(x, a) \\
    =& \sum_{t = 0}^{\infty} \sum_{k=0}^{\infty} \mathbb E\left[ w(t) \gamma^k \qf(S_{t+k}, A_{t + k}) \psi(S_{t+k}, A_{t + k}) \middle | \theta \right]\\
    =& \sum_{t = 0}^{\infty} \sum_{i=t}^{\infty} \mathbb E\left[ w(t) \gamma^{i - t} \qf(S_i, A_i) \psi(S_i, A_i) \middle | \theta \right],\\
\end{align*}
by substitution of the variable $i = t + k$.
Continuing, we can move the summation inside the expectation and reorder the summation:

\begin{align*}
    \sum_{s \in \mathcal S^\circ} d^\theta_\gamma(s) \frac{\partial}{\partial \theta } \vf(s)
    &= \mathbb E \left[\sum_{t = 0}^{\infty} \sum_{i = t}^{\infty} w(t) \gamma^{i - t} \qf(S_i, A_i) \psi(S_i, A_i) \bigg| \theta \right] \\
    &= \mathbb E \left[\sum_{i = 0}^{\infty} \sum_{t = 0}^i w(t) \gamma^{i - t} \qf(S_i, A_i) \psi(S_i, A_i) \bigg| \theta \right] \\
    &= \mathbb E \left[\sum_{i = 0}^{\infty} \qf(S_i, A_i) \psi(S_i, A_i) \sum_{t = 0}^i w(t) \gamma^{i - t} \bigg| \theta \right].
\end{align*}
In order to derive $\jq$, we simply need to choose a $w(t)$ such that $\forall i: \sum_{t = 0}^i w(t) \gamma^{i - t} = 1$.
This is satisfied by the choice: $w(t) = 1$ if $t = 0$, and  $1 - \gamma$ otherwise.
This trivially holds for $i=0$, as $w(0) \gamma^0 = (1) (1) = 1$.
For $i > 0$:

\begin{align*}
    \sum_{t = 0}^i w(t) \gamma^{i - t} &= w(0) \gamma^i + \sum_{t = 1}^i w(t) \gamma^{i - t} \\
        &= \gamma^i + \sum_{t = 1}^i (1 - \gamma) \gamma^{i - t} \\
        &= \gamma^i + \underbrace{\sum_{t = 1}^i (\gamma^{i - t} - \gamma^{i - t + 1})}_\text{telescoping series} \\
        &= \gamma^i + \gamma^{i - i} - \gamma^{i - 1 + 1} \\
        &= \gamma^i + 1 - \gamma^i \\
        &= 1.
\end{align*}

Thus, for this choice of $w(t)$:

\begin{align*}
    \sum_{s \in \mathcal S^\circ} d^\theta_\gamma(s) \frac{\partial}{\partial \theta } \vf(s) &= \mathbb E \left[\sum_{i = 0}^{\infty} \qf(S_i, A_i) \psi(S_i, A_i) \bigg| \theta \right] \\
    &= \jq.
\end{align*}

Finally, we see that this choice of $w(t)$ also gives us the expression for $\df$ stated in Lemma \ref{lem:biased-dist}:

\begin{align*}
    \df(s) &= \sum_{t = 0}^{\infty} w(t) \Pr(S_t = s | \theta) \\
    &= \underbrace{w(0)}_{1} \underbrace{\Pr(S_0 = s | \theta)}_{d_0(s)} + \sum_{t = 1}^{\infty} \underbrace{w(t)}_{1 - \gamma} \Pr(S_t = s | \theta) \\
    &= d_0(s) + (1 - \gamma) \sum_{t = 1}^{\infty} \Pr(S_t = s|\theta).
\end{align*}

\subsection*{Details for the Proof of Theorem \ref{thm:no-grad}}

We give the intermediate calculations for the example and policy parameterization specified in the proof of Theorem \ref{thm:no-grad}.
First, we compute $\df$ in terms of $\theta$ for each nonterminal state using the definition of the MDP and $\pi$:

\begin{align*}
    \df(s_1)&= 1 \\
    \df(s_2)&= (1 - \gamma) \Pr(S_1 = s_2) \\
            &= (1 - \gamma) \pi(s_1, a_1) \\
            &= (1 - \gamma) \sigma(\theta_1) \\
\end{align*}

Next, we compute $\vf$ in each state in terms of $\theta$.
Note that $\qf(s_1, a_1) = \gamma \vf(s_2)$ because taking $a_1$ in $s_1$ leads to $s_2$ and has zero reward.

\begin{align*}
    \vf(s_2) &= \pi(s_2, a_1) \qf(s_2, a_1) + \pi(s_2, a_2) \qf(s_2, a_2) \\
             &= \pi(s_2, a_1) (1) + \pi(s_2, a_2) (0) \\
             &= \sigma(\theta_2) \\
    \vf(s_1) &= \pi(s_1, a_1) \qf(s_1, a_1) + \pi(s_1, a_2) \qf(s_1, a_2) \\
             &= \pi(s_1, a_1) (\gamma \vf(s_2)) + \pi(s_1, a_2) (0) \\
             &= \gamma \sigma(\theta_1) \sigma(\theta_2)
\end{align*}

Recall that by the definition of our policy and substitution, we have:

\begin{align*}
    \frac{\partial \pi^\theta(s_1,a_1)}{\partial \theta_1} &= \frac{\partial \sigma(\theta_1)}{\partial \theta_1} \\
    \frac{\partial \pi^\theta(s_2,a_1)}{\partial \theta_2} &= \frac{\partial \sigma(\theta_2)}{\partial \theta_2} 
\end{align*}

Next, we compute each partial derivative, $\partial \vf(s) / \partial \theta_i$:

\begin{align*}
    \frac{\partial \vf(s_1)}{\partial \theta_1} &= \frac{\partial}{\partial \theta_1} \gamma \sigma(\theta_1) \sigma(\theta_2) \\
    &= \gamma \sigma(\theta_2) \frac{\partial \sigma(\theta_1)}{\partial \theta_1} \\
    \frac{\partial \vf(s_1)}{\partial \theta_2} &= \frac{\partial}{\partial \theta_2} \gamma \sigma(\theta_1) \sigma(\theta_2) \\
    &= \gamma \sigma(\theta_1) \frac{\partial \sigma(\theta_2)}{\partial \theta_2} \\
    \frac{\partial \vf(s_2)}{\partial \theta_1} &= \frac{\partial \sigma(\theta_2)}{\partial \theta_1} \\
    &= 0 \\
    \frac{\partial \vf(s_2)}{\partial \theta_2} &= \frac{\partial\sigma(\theta_2)}{\partial \theta_2} 
\end{align*}

With the necessary components in place, we can apply Lemma \ref{lem:biased-dist} to compute each component of $F_\gamma$:

\begin{align*}
    F_{\gamma,1}(\theta) &= \underbrace{\df(s_1)}_1 \frac{\partial \vf(s_1)}{\partial \theta_1} + \df(s_2) \underbrace{\frac{\partial \vf(s_2)}{\partial \theta_1}}_0 \\
    &= \gamma \sigma(\theta_2) \frac{\partial \sigma(\theta_1)}{\partial \theta_1} \\
    F_{\gamma,2}(\theta) &= \df(s_1) \frac{\partial \vf(s_1)}{\partial \theta_2} + \df(s_2) \frac{\partial \vf(s_2)}{\partial \theta_2} \\
    &= \gamma \sigma(\theta_1) \frac{\partial \sigma(\theta_2)}{\partial \theta_2} + (1 - \gamma) \sigma(\theta_1) \frac{\partial\sigma(\theta_2)}{\partial \theta_2} \\ &= \sigma(\theta_1) \frac{\partial\sigma(\theta_2)}{\partial \theta_2}
\end{align*}

Next, we compute the off-diagonal entries of the Jacobian of $F_\gamma$:

\begin{align*}
    \frac{\partial}{\partial \theta_2} \left( F_{\gamma,1}(\theta) \right) &= \frac{\partial}{\partial \theta_2} \left( \gamma \sigma(\theta_2) \frac{\partial \sigma(\theta_1)}{\partial \theta_1} \right) \\
    &= \gamma \frac{\partial \sigma(\theta_1)}{\partial \theta_1} \frac{\partial \sigma(\theta_2)}{\partial \theta_2} \\
    \frac{\partial}{\partial \theta_1} \left( F_{\gamma,2}(\theta) \right) &= \frac{\partial}{\partial \theta_1} \left( \sigma(\theta_1) \frac{\partial \sigma(\theta_2)}{\partial \theta_2} \right) \\
    &= \frac{\partial \sigma(\theta_1)}{\partial \theta_1} \frac{\partial \sigma(\theta_2)}{\partial \theta_2}
\end{align*}

Thus we see that the following holds for any $\theta_1$ and $\theta_2$:

\begin{equation}
    \forall \gamma \in [0,1): \frac{\partial}{\partial \theta_2} \left( F_{\gamma,1}(\theta) \right) \neq \frac{\partial}{\partial \theta_1} \left( F_{\gamma,2}(\theta) \right).
\end{equation}

For completeness, the diagonal entries of the Jacobian are:

\begin{align*}
    \frac{\partial}{\partial \theta_1} \left( F_{\gamma,1}(\theta) \right) &= \frac{\partial}{\partial \theta_1} \left( \gamma \sigma(\theta_2) \frac{\partial \sigma(\theta_1)}{\partial \theta_1} \right) \\
    &= \gamma \sigma(\theta_2) \frac{\partial^2 \sigma(\theta_1)}{\partial \theta_1^2} \\
    \frac{\partial}{\partial \theta_2} \left( F_{\gamma,2}(\theta) \right) &= \frac{\partial}{\partial \theta_2} \left( \sigma(\theta_1) \frac{\partial \sigma(\theta_2)}{\partial \theta_2} \right) \\
    &= \sigma(\theta_1) \frac{\partial^2 \sigma(\theta_2)}{\partial \theta_2^2}.
\end{align*}

All entries of the Jacobian are continuous on $\mathbb R^2$, and the off-diagonal entries differ at every point for $\gamma\in[0,1)$ because $\sigma'(x)>0$ for every finite $x$.
If $F_\gamma=\nabla J$ for a differentiable function $J$ on a nonempty open set $U$, these continuous derivatives would make $J$ twice continuously differentiable on $U$.
The unequal cross derivatives then contradict Theorem \ref{thm:clairaut-shwarz} on $U$.
This proves the claimed local, and hence global, nonexistence of an objective.

\input{extended-lit-review.tex}

%% file: extended-lit-review.tex
\subsection*{Additional Literature Review}

Here we include additional analyses of papers associated with algorithms from \texttt{stable-baselines} that were not discussed in the main text.
However, several of the papers included claims that we argue are nevertheless misleading, some more so than others.
Due to space limitations, we are forced to omit our analysis of DDPG and TRPO.

\subsection*{ACER (ICLR 2017) (299 Citations)}

ACER \cite{wang2017sample} combines experience replay with actor-critic methods, improving sample efficiency. The background section contains the following text:

\begin{quote}
    The parameters $\theta$ of the differentiable policy $\pi_\theta(a_t|x_t)$ can be updated using the discounted approximation to the policy gradient \cite{sutton2000policy}, which borrowing notation from \citet{schulman2015high}, is defined as:

    \begin{equation}
        g = \mathbb E_{x_0:\infty, a_0:\infty} \left[\sum_{t \geq 0} A^\pi(x_t, a_t) \nabla_\theta \log \pi_\theta(a_t|x_t) \right].
    \end{equation}    
    
    Following Proposition 1 of \citet{schulman2015high}, we can replace $A^\pi(x_t, a_t)$ in the above expression with the state-action value $Q^\pi(x_t, a_t)$, the discounted return $R_t$, or the temporal difference residual \\ $r_t + \gamma V^\pi(x_{t+1}) - V^\pi(x_t)$, without introducing bias.
\end{quote}
Under our on-policy episodic convention and with exact discounted advantages, the displayed $g$ equals $\jq$.
This identifies the displayed expression, not every update in ACER's replay-based implementation.
The authors claim that any of the given choices may be used ``without introducing bias.''
We would argue that a naive reader is likely to assume that this means that $g$ is unbiased with respect to some objective.
The authors were actually making the subtler point that the given choices do not introduce bias relative to the choice of $A^\pi$,  which is itself biased.
This claim is not erroneous, but at the same time, information is left out that is important for a clear understanding.
The correctness hinges on ambiguity rather than precision, and a reader is likely to come away with the opposite impression.
For this reason, we argue that this section is still misleading.

\subsection*{PPO (arXiv 2017, 1769 citations)}

Proximal Policy Optimization (PPO) \cite{schulman2017proximal} is arguably the most popular deep actor-critic algorithm at this time due to its speed and sample efficiency.
The paper contains the text:

\begin{quote}
    Policy gradient methods work by computing an estimator of the policy gradient and plugging it into a stochastic gradient ascent algorithm. The most commonly used gradient estimator has the form

    \begin{equation}
        \hat g = \hat{\mathbb E}[\nabla_\theta \log \pi_\theta(a_t|s_t) \hat A_t]
    \end{equation}
    
    where $\pi_\theta$ is a stochastic policy and $\hat A_t$ is an estimator of the advantage function at timestep $t$. Here the expectation $\hat{\mathbb E}[\dots]$ indicates the empirical average over a finite batch of samples, in an algorithm that alternates between sampling and optimization.
    

    
    

\end{quote}
In this case, the authors considered a very specific setup where an algorithm ``alternates between sampling and optimization.''
For PPO, they subsequently construct a modified surrogate objective for each batch and optimize it over multiple epochs of minibatch updates before collecting a new batch.
The final algorithm does follow a direction resembling $\hat g$, with a number of optimization tricks.
For this reason, we did not consider the claims above made to be misleading.
The existence of a gradient for a fixed-batch surrogate is consistent with Theorem \ref{thm:no-grad}, which concerns the field obtained when the trajectory distribution and exact action values vary with the policy.
However, the question of whether $\jq$ is a gradient is sidestepped rather than addressed directly by these fixed-batch claims.

\subsection*{TD3 (ICML 2018, 247 citations)}

Twin Delayed Deep Deterministic policy gradient (TD3) \citet{fujimoto2018addressing} is another paper in the DDPG family.
The paper was published concurrently with SAC and contains similar enhancements.
It contains the text:

\begin{quote}
    The return is defined as the discounted sum of rewards $R_t = \sum_{i=t}^T\gamma^{i - t}r(s_i, a_i)$, where $\gamma$ is a discount factor determining the priority of short-term rewards.
    In reinforcement learning, the objective is to find the optimal policy $\pi_\phi$, with parameters $\phi$, which maximizes the expected return $J(\phi) = \mathbb E_{s_i \sim p_\pi, a_i\sim\pi}[R_0]$.
    For continuous control, parameterized policies $\pi_\phi$ can be updated by taking the gradient of the expected return $\nabla_\phi J(\phi)$.
    In actor-critic methods, the policy, known as the actor, can be updated through the deterministic policy gradient algorithm:
    
    \begin{equation}
        \nabla_\phi J(\phi) = \mathbb E_{s \sim p_\pi}[\nabla_a Q^\pi(s, a)|_{a = \pi(s)} \nabla_\phi \pi_\phi(s)].
    \end{equation}

    $Q^\pi(s, a) = \mathbb E_{s_i \sim p_\pi, a_i\sim\pi}[R_t|s, a]$, the expected return when performing action $a$ in state $s$ and following $\pi$ after, is known as the critic or the value function.
\end{quote}

The authors did not define $p_\pi$, leaving the above expression ambiguous.
In the original DDPG paper, it was defined as the discounted state distribution.
It is misused here in the definitions of $J$ and $Q$, in that it's not clear what role the discounted state distribution plays in the definition of $Q$ in either case.
The algorithm eventually computes an actor update by averaging over samples drawn from a replay buffer: $N^{-1} \Sigma \nabla_a Q_{\theta_1}(s, a) |_{a=\pi_\phi(s)}\nabla_\phi \pi_\phi(s)$.
This uses the replay-buffer state distribution and is not an estimator of the episodic field $\jq$ without additional assumptions.
Identifying it with the gradient of the stated discounted objective requires specifying the relationship between the replay distribution and the state weighting in the policy gradient theorem.